\PassOptionsToPackage{unicode}{hyperref}
\PassOptionsToPackage{hyphens}{url}
\PassOptionsToPackage{dvipsnames,svgnames,x11names}{xcolor}
\documentclass[
  12pt]{article}

\usepackage{amsmath,amssymb}
\usepackage{iftex}
\ifPDFTeX
  \usepackage[T1]{fontenc}
  \usepackage[utf8]{inputenc}
  \usepackage{textcomp} 
\else 
  \usepackage{unicode-math}
  \defaultfontfeatures{Scale=MatchLowercase}
  \defaultfontfeatures[\rmfamily]{Ligatures=TeX,Scale=1}
\fi
\usepackage{lmodern}
\ifPDFTeX\else  
\fi
\IfFileExists{upquote.sty}{\usepackage{upquote}}{}
\IfFileExists{microtype.sty}{
  \usepackage[]{microtype}
  \UseMicrotypeSet[protrusion]{basicmath} 
}{}
\makeatletter
\@ifundefined{KOMAClassName}{
  \IfFileExists{parskip.sty}{%
    \usepackage{parskip}
  }{
    \setlength{\parindent}{0pt}
    \setlength{\parskip}{6pt plus 2pt minus 1pt}}
}{
  \KOMAoptions{parskip=half}}
\makeatother
\usepackage{xcolor}
\makeatletter
\ifx\paragraph\undefined\else
  \let\oldparagraph\paragraph
  \renewcommand{\paragraph}{
    \@ifstar
      \xxxParagraphStar
      \xxxParagraphNoStar
  }
  \newcommand{\xxxParagraphStar}[1]{\oldparagraph*{#1}\mbox{}}
  \newcommand{\xxxParagraphNoStar}[1]{\oldparagraph{#1}\mbox{}}
\fi
\ifx\subparagraph\undefined\else
  \let\oldsubparagraph\subparagraph
  \renewcommand{\subparagraph}{
    \@ifstar
      \xxxSubParagraphStar
      \xxxSubParagraphNoStar
  }
  \newcommand{\xxxSubParagraphStar}[1]{\oldsubparagraph*{#1}\mbox{}}
  \newcommand{\xxxSubParagraphNoStar}[1]{\oldsubparagraph{#1}\mbox{}}
\fi
\makeatother

\usepackage{longtable,booktabs,array}
\usepackage{calc} 
\usepackage{etoolbox}
\makeatletter
\patchcmd\longtable{\par}{\if@noskipsec\mbox{}\fi\par}{}{}
\makeatother
\IfFileExists{footnotehyper.sty}{\usepackage{footnotehyper}}{\usepackage{footnote}}
\makesavenoteenv{longtable}
\usepackage{graphicx}
\makeatletter
\def\maxwidth{\ifdim\Gin@nat@width>\linewidth\linewidth\else\Gin@nat@width\fi}
\def\maxheight{\ifdim\Gin@nat@height>\textheight\textheight\else\Gin@nat@height\fi}
\makeatother
\setkeys{Gin}{width=\maxwidth,height=\maxheight,keepaspectratio}
\makeatletter
\def\fps@figure{htbp}
\makeatother

\makeatletter
\@ifpackageloaded{caption}{}{\usepackage{caption}}
\AtBeginDocument{%
\ifdefined\contentsname
  \renewcommand*\contentsname{Table of contents}
\else
  \newcommand\contentsname{Table of contents}
\fi
\ifdefined\listfigurename
  \renewcommand*\listfigurename{List of Figures}
\else
  \newcommand\listfigurename{List of Figures}
\fi
\ifdefined\listtablename
  \renewcommand*\listtablename{List of Tables}
\else
  \newcommand\listtablename{List of Tables}
\fi
\ifdefined\figurename
  \renewcommand*\figurename{Figure}
\else
  \newcommand\figurename{Figure}
\fi
\ifdefined\tablename
  \renewcommand*\tablename{Table}
\else
  \newcommand\tablename{Table}
\fi
}
\@ifpackageloaded{float}{}{\usepackage{float}}
\floatstyle{ruled}
\@ifundefined{c@chapter}{\newfloat{codelisting}{h}{lop}}{\newfloat{codelisting}{h}{lop}[chapter]}
\floatname{codelisting}{Listing}

\makeatother
\makeatletter
\@ifpackageloaded{caption}{}{\usepackage{caption}}
\@ifpackageloaded{subcaption}{}{\usepackage{subcaption}}
\makeatother

\ifLuaTeX
  \usepackage{selnolig}  
\fi
\usepackage[]{natbib}
\usepackage{bookmark}

\IfFileExists{xurl.sty}{\usepackage{xurl}}{} 
\hypersetup{
  pdftitle={Title},
  pdfauthor={Author 1; Author 2},
  pdfkeywords={3 to 6 keywords, that do not appear in the title},
  colorlinks=true,
  linkcolor={blue},
  filecolor={Maroon},
  citecolor={Blue},
  urlcolor={Blue},
  pdfcreator={LaTeX via pandoc}}

\newcommand{\anon}{1}

\usepackage{mathrsfs} 
\usepackage{amsmath} 
\usepackage{bm} 
\usepackage{amsthm} 
\newtheorem{theorem}{Theorem}
\newtheorem{lemma}{Lemma} 
\usepackage{booktabs} 
\usepackage{array} 
\usepackage{tabularx} 
\usepackage{booktabs}

\begin{document}

\def\spacingset#1{\renewcommand{\baselinestretch}%
{#1}\small\normalsize} \spacingset{1}


\if1\anon
{
  \title{\bf Multiclass Calibration Assessment and Recalibration of Probability Predictions via the Linear Log Odds Calibration Function}
  \author{Amy Vennos\footnote{corresponding author} , Xin Xing, and Christopher T. Franck\thanks{This material is based upon work supported, in whole or in part, by the U.S. Department of Defense (DoD) through the Office of the Under Secretary of Defense for Acquisition and Sustainment (OUSD(A\&S)) and the Office of the Under Secretary of Defense for Research and Engineering (OUSD(R\&E)) under Contract HQ003424D0023. The Acquisition Innovation Research Center (AIRC) is a multi-university partnership led and managed by the Stevens Institute of Technology through the Systems Engineering Research Center (SERC) – a federally funded University Affiliated Research Center. Any views, opinions, findings and conclusions or recommendations expressed in this material are those of the author(s) and do not necessarily reflect the views of the United States Government (including the DoD and any government personnel).} \\
    Department of Statistics, Virginia Polytechnic Institute and State University} 
  \maketitle
} \fi

\if0\anon
{
  \bigskip
  \bigskip
  \bigskip
  \begin{center}
    {\LARGE\bf Title}
\end{center}
  \medskip
} \fi

\bigskip
\begin{abstract}
Machine-generated probability predictions are essential in modern classification tasks such as image classification. A model is well calibrated when its predicted probabilities correspond to observed event frequencies. Despite the need for multicategory recalibration methods, existing  methods are limited to (i) comparing calibration between two or more models rather than directly assessing the calibration of a single model, (ii) requiring under-the-hood model access, e.g., accessing logit-scale predictions within the layers of a neural network, and (iii) providing output which is difficult for human analysts to understand. To overcome (i)-(iii), we propose Multicategory Linear Log Odds (MCLLO) recalibration, which (i) includes a likelihood ratio hypothesis test to assess calibration, (ii) does not require under-the-hood access to models and is thus applicable on a wide range of classification problems, and (iii) can be easily interpreted. We demonstrate the effectiveness of the MCLLO method through simulations and three real-world case studies involving image classification via convolutional neural network, obesity analysis via random forest, and ecology via regression modeling. We compare MCLLO to four comparator recalibration techniques utilizing both our hypothesis test and the existing calibration metric Expected Calibration Error to show that our method works well alone and in concert with other methods.
\end{abstract}

\noindent%
\vfill

\newpage
\spacingset{1.8} 

\section{Introduction}\label{sec-intro}

Multicategory probability predictions arise from statistical and machine classifiers and estimate the probability that objects belong to each of $c$ categories. For example, an image classifier may be used to predict which animals or vehicles are contained within a set of images. Multicategory predictions are used widely in image classification \citep{sambyal_2023_imageclassificationcalibration, Rajaraman_2022_calibration_images, Coz_imageclassification, Kuppers_2020_calibrationimageclassification} including for satellite images \citep{karalas2016remotesensing, wanremotesensing}, medical diagnoses \citep{HameedSkinCancerDetection, Mi2021CancerClassification}, geological hazard mapping \citep{GoetzGeologicalHazard},  speech recognition \citep{Ganapathiraju_2004_SVMspeechrecognition}, remote sensing \citep{maxwellremotesensing}, defense applications \citep{machado2023defense}, and sorting through gene expression datasets for cancer research \citep{Panca2017multiclasscancerresearch1, Yeung2003cancerresearch2, Ramaswamy2001multiclasscancerresearch3}.  

The value of probability predictions to decision makers depends on \textit{calibration}, which describes how well empirical observations agree with the claimed probability predictions. If a model is well calibrated, then the probability predictions in all categories agree with the observed frequencies of observations in each category. For example, suppose a diagnostic model predicts a 40\% chance that a patient has cancer. Clinical risk is only properly conveyed by the model if 4 in 10 patients who are claimed to have a 40\% chance of cancer actually have cancer. More generally, we say that a model’s predictions are well calibrated if this property holds for all predictions between 0\% and 100\%. Perhaps surprisingly, there are no guarantees that probability predictions from statistical and machine learning models are well calibrated, especially when predictions are made outside of the data set used to train the model. We want to know whether predictions are well calibrated, and if they aren't calibrated, we need methods to recalibrate them. Recalibration methods adjust probability predictions so they are calibrated. 

Figure \ref{fig:aretheseplanes} provides an example of calibration in the context of multicategory image classification. This figure shows two images (indices 5037 and 4526) of planes from the CIFAR-10 data set \citep{Krizhevsky2009cifar}. The image on the left is more clearly a plane than the image on the right.  A neural net classified both images as planes and outputted confidence scores that reflect the neural net's certainty in each of the ten possible labels. Note that 0-1 probability scale predictions outputted by machine learning algorithms are sometimes called \textit{confidence scores}. We include the top five outputted confidence scores from the Visual Geometry Group  neural net (VGGNet) \citep{simonyan_2014_vggnet} followed by recalibrated versions of those predictions furnished by out Multicategory Linear Log Odds (MCLLO) method, which is described in Section \ref{sec:Methods}. In this example, the image on the left received a confidence score of 1.00 for plane and a confidence score of 0.00 for all other categories. This image is obviously a plane. If we assess the empirical rate at which objects with the same confidence scores are actually planes, it is within rounding error of 100\%, thus the recalibrated probability is also 1.00. By contrast, for the more ambiguous image on the right, the neural net assigns a confidence score of 0.94. However, on the basis of available training data, scores of 94\% as in this case only truly contain planes in 80\% of cases. The distinction between a 94\% probability of an event versus an 80\% actual probability of the event would be crucial in many decision making tasks. Our MCLLO method can be used to detect the uncalibrated nature of these predictions and adjust them so that e.g., a confidence score of 94\% in this example actually corresponds to an empirical rate of actually being a plane of 80\%. 

\begin{figure}

\centering{

\includegraphics[width=4.1in,height=\textheight]{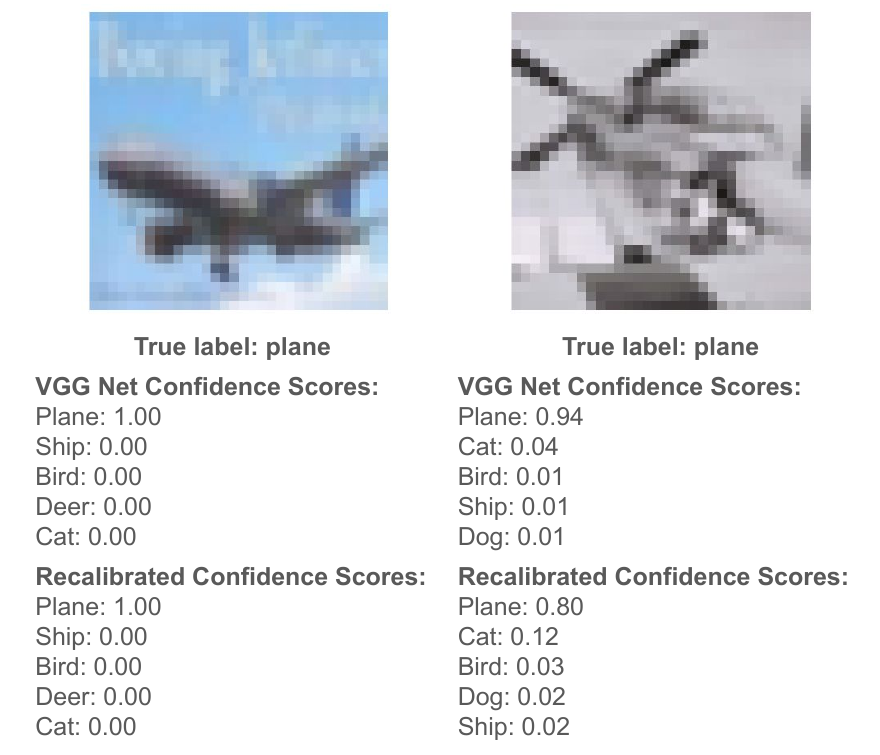}

}

\caption{\label{fig:aretheseplanes}Two images of observed label ``plane" from the CIFAR-10 data set. A neural net outputted confidence scores of a set of ten labels for each image and reported high confidence scores for the label of ``plane" for both images. However, the image on the right is less recognizable as a ``plane" than the image on the left. Direct recalibration via MCLLO as described in Section (\ref{sec:Methods}) adjusts the probability of each label to correspond better with the rates that planes occur given data such as these images.}

\end{figure}

Much work has been done in ensuring machine learning models are producing well-calibrated predictions  \citep{Niculescu_machinelearningcalibration, guo2017calibration, kumar_machinelearningcalibration}. In the binary setting, much research has been published in improving calibration, including  empirical binning \citep{zadroznyelkan2001binning, Pakdaman_Naeini2015binarybinning}, isotonic regression \citep{Barlow1972isotonicregression, Naeini2016isotonicregression}, Platt scaling \citep{platt1999PlattScaling}, and Beta calibration \citep{kull2017betacalibration}. In our literature review of recalibration methods, we were surprised to learn that many efforts to improve calibration in binary cases do not appear to have a straightforward extension to the multicategory setting: for example, the multicategory setting introduces varieties in the definition of calibration. The strongest of these varieties is \textit{multiclass} calibration, which considers all probability predictions for every category and every observation, and is preferred over other methods of calibration \citep{springerpaper}.  

Two popular methods of multiclass recalibration are temperature scaling and vector scaling \citep{hinton2015temperaturescaling, guo2017calibration}, which are extensions of Platt scaling that linearly scale the  \textit{logits}, or scores outputted from the last layer of the neural net before the softmax function, instead of probability predictions.  Since the softmax function is not invertible, the logits in the layer before the softmax function cannot be recovered. Consequently, temperature scaling and vector scaling are limited in their scope in that the methods can only be applied when the logits are available. This may involve under-the-hood access to, say, a convolutional neural network model in order to access and recalibrate the model logits in the appropriate layer of the network. We show an example of this in Section \ref{section:data:cifar}. Some methods, such as random forests, produce confidence scores by vote counting trees and do not have natural logits with which to employ temperature and vector scaling. These scaling methods are thus unusable for probability scale predictions. We show such an example in Section \ref{section:data:obesity}. Additionally, certain shifts of the logits that produce the same probability predictions are recalibrated differently by these methods. Our MCLLO method overcomes this issue by operating directly on probability scale predictions, which are available in virtually all modern classification techniques. Unlike temperature and vector scaling, MCLLO methods do not require under-the-hood access to models in order to implement recalibration. 


Another shortcoming among existing methods is that there are no recalibration techniques that include hypothesis testing for calibration status. This brings about a limitation in that many recalibration techniques rely on outside measures of calibration, such as the Hosmer Lemeshow (HL) test (which can test for calibration but does nor prescribe a recalibration approach) \citep{Fagerland2008classwisehypothesistesting}, a visualization of calibration called reliability diagrams \citep{Vaicenavicius2019multiclasshypothesistest}, \cite{widmann2019multiclasshypothesistest}'s matrix-valued kernel-based quantification of calibration, and the popular Expected Calibration Error (ECE) \citep{guo2017calibration, Naeini_2015_bayesianbinning} to determine calibration status. These measures of calibration are limited in that they cannot simultaneously determine with a prescribed significance level if a model is well-calibrated and also prescribe a recalibration approach if necessary. Methods such as the ECE score are useful for comparing competing models (lower score is better) but offer less insight as to whether a single model is performing adequately. Many practitioners may appreciate the hypothesis testing framework that can make calibration status more clear on the basis of a single model's predictions, thus informing the need to recalibrate or not.

Additionally, binning-based metrics of calibration like ECE and Maximum Calibration Error \citep{guo2017calibration} are particularly limited by their dependence on a user-specified number of bins, as a different number of bins can impact the resulting calibration metric greatly. For example, in some applications, ECE and MCE can produce either the same metric or vastly different results depending on the chosen bin size for the same data set \citep{Reinke_MCC_ECE_comparison}. See the discussion in Section \ref{sec:discussion} for commentary on how the number of bins affects ECE results. Heuristics for the number of bins in calculating ECE vary, ranging from a recommendation of 10 to 20 bins \citep{nixon2019ecebins}, to using the square root of the number of observations \citep{posocco_ECE_squareroot}. By contrast, our MCLLO method works directly on the probability scale with no need for binning. We will adopt the square root rule for this paper when comparing binning-based methods such as ECE with the MCLLO method.

A recalibration technique with built-in hypothesis testing is advantageous to decision-makers in several ways. First, a recalibrator with built-in hypothesis testing guarantees inclusion of the identity map, a property noted by  \cite{Vaicenavicius2019multiclasshypothesistest} that describes recalibration techniques that identify well-calibrated probability predictions. Recalibrators that do not include the identity function will attempt to recalibrate well-calibrated data. For example, \cite{xudavoine} implement a multiclass calibration approach that utilizes Maximum Likelihood Estimates (MLEs), but their approach does not include the identity map and thus would at least slightly perturb even well calibrated probabilities.  Second, even though some existing recalibrators without the capability of hypothesis testing include the identity map, including a built-in hypothesis testing based metric of calibration controls the type one error rate and gains power as sample size or departure from calibration grows and is not dependent on a user-specified number of bins, which are properties that existing measures of calibration lack. Finally, a recalibration method that does not depend on a user-specified bin size and that applies to all confidence outputs by machine learning models and data models allows decision-makers to better trust a wider scope of probability predictions. 

Through our simulation study and case studies, we demonstrate that MCLLO recalibration can (1) identify well-calibrated data, (2) apply to a wider scope of machine-generated predictive outputs than temperature scaling or vector scaling, and (3) improve calibration of uncalibrated data. The effectiveness of recalibration is shown not only using our LRT, but also via the existing measures of calibration ECE and reliability diagrams. 

Furthermore, we compare MCLLO recalibration to four existing recalibration techniques. We implement an extension of binning described by \cite{guo2017calibration}, and include temperature and vector scaling as described by this paper, which otherwise compares existing methods of recalibration instead of proposing a new method of recalibration. We also implement a multiclass calibration technique developed by \cite{xudavoine}. We implement and assess these techniques in three case studies, showing virtues and drawbacks of all methods under comparison.

This paper is organized as follows. Section \ref{sec:Methods} introduces MCLLO Recalibration and its properties. Section \ref{section:Simulation} demonstrates the capabilities MCLLO recalibration and LRT through a simulation study.   Section \ref{section:casestudies} applies MCLLO recalibration to image classification, obesity, and ecology datasets and compares its functionality to four comparator recalibration techniques. Section \ref{sec:discussion} provides a discussion of our proposed recalibration technique and results.

\section{Multicategory Linear Log Odds Methodology} \label{sec:Methods}

We address the multicategory classification problem with $c>1$ classes. Each observation $\mathbf{Y}_i$ belongs in one and only one class. Thus we consider the random vector $\mathbf{Y}_{i \cdot} \sim MN ( \mathbf{p}_{i, c \times 1} \in \mathbb{R}^c, \text{size} = 1)$, where $\mathbf{p}_{i, c \times 1}$ is a collection of predictions that the observation will assume the value of each category. The goal of calibration assessment and recalibration is to ensure that probability predictions $\mathbf{{x_i}}$ correspond to the actual event rates $\mathbf{p}_i$. If necessary the   $\mathbf{{x_i}}$ can be adjusted, i.e., recalibrated to resemble the $\mathbf{p}_i$ so that the recalibrated predictions ultimatly describe the actual events  $\mathbf{Y}_i$.

This paper considers the possibilities $\mathbf{p}_{c \times 1, i} = \mathbf{x}_{i \cdot}$, the original probability predictions, versus $\mathbf{p}_{c \times 1, i} = \mathbf{g}_{i \cdot}$, recalibrated predictions. The original and recalibrated predictions are stored in matrices $\mathbf{X}$ and $\mathbf{G}$, respectively.  In this section, we consider the possibility that the recalibrated set of predictions $\mathbf{G}$ describing observations $\mathbf{Y}$ differs from the original predictions $\mathbf{X}$ according to a linear function on the log odds scale.

In our multicategory framework, predictor's odds are defined between each category and the \textit{baseline} category. We follow the convention of setting the baseline category as the last category  $c$ \citep{agresti2006multinomiallogisticregression}.  In the MCLLO model, the analysis is not technically invariant to the choice of baseline category, although empirically we show that the baseline choice hsa a small effect on the analysis. See supplementary material for more details. 

We propose a transformation that maps each prediction $x_{ij}$ to a recalibrated prediction $g_{ij}$ for each observation $i \in \{ 1 , \ldots , n\}$, and category $j \in \{ 1 , \ldots , c\}$. This is done via a  transformation of the log odds of each event probability prediction $x_{ij}$ versus the probability prediction corresponding the baseline category $x_{ic}$, denoted $\log \left( \frac{x_{ij}}{x_{ic}} \right)$. The log odds of event probability predictions are mapped to the log odds of the  recalibrated predictions $g_{ij}$ and $g_{ic}$, denoted $\log \left( \frac{g_{ij}}{g_{ic}} \right)$. The mapping is a linear function involving shift and scale parameters.  Define two length $c-1$ vectors $\bm{\delta}$ and $ \bm{\gamma}$ that contain the parameter values for the recalibration shift and scale, respectively. More details about the interpretation of $\bm{\delta}$ and $\bm{\gamma}$ can be found in the supplementary material. The goal of recalibration is to find optimal parameter estimates $\hat{\bm{\delta}}$ and $\hat{\bm{\gamma}}$ so that the recalibrated predictions better reflect the observed frequencies in $\mathbf{Y}$. We obtain $\hat{\bm{\delta}}$ and $\hat{\bm{\gamma}}$ by maximizing a likelihood that associates probability predictions $\mathbf{{x_i}}$ with event data $\mathbf{Y}_i$. Mathematically, this shift and scale follows the system of baseline-category log odds formulas defined by

\begin{equation}
    \label{logodds_linear_model}
    \log \left( \frac{g_{ij}}{g_{ic}} \right) = \log (\delta_j) + \gamma_j log \left( \frac{x_{ij}}{x_{ic}} \right),
\end{equation}
for $i\in \{ 1, \ldots , n \}$ and $j \in \{ 1, \ldots , c-1 \} $. The recalibrated probability prediction for the baseline category follows from the axioms of probability, shown in Equation (\ref{eqn:logoddsmodelaxiom}).

\begin{equation} \label{eqn:logoddsmodelaxiom}
    \sum_{j=1}^c g_{ij} = 1.
\end{equation}

\noindent Our MCLLO approach involves $c$ categories and therefore utilizes a system of $c$ equations with $2 (c-1)$ parameters.

 Equation (\ref{logodds_linear_model}) is called  the Multicategory Linear Log Odds (MCLLO) recalibration formula. Manipulating this formula with the constraint provided by Equation (\ref{eqn:logoddsmodelaxiom}) reveals a recalibrator for multicategory predictions. Another advantage of MCLLO recalibration is the availability of hypothesis testing, as parameter values of $\bm{\delta} = \bm{\gamma} = \mathbf{1}$ represent well calibrated predictions that do not require recalibration. MCLLO recalibration is therefore valuable to any multicategory classification application, as a researcher can (1) test for well-calibrated probability predictions via the methodology in Section \ref{section:LRT}, and if the predictions are not well-calibrated, (2) produce optimally calibrated probability predictions via equations provided in Section \ref{section:recalibration}.

\subsection{Analytic Form of Recalibrated Probability Predictions} \label{section:recalibration}

We algebraically re-express Equations (\ref{logodds_linear_model}) and (\ref{eqn:logoddsmodelaxiom}) to obtain Equation (\ref{g_equation1}), which provides the formula that produces recalibrated probability predictions $g_{ij}$ as a function of uncalibrated predictions and parameter vectors  $\bm{\gamma}$ and $\bm{\delta}$. The recalibrated probability predictions take the form

\begin{equation}
    g_{ij} = \frac{\delta_j x_{ij}^{\gamma_j}}{x_{ic}^{\gamma_j} + \sum_{m=1}^{c-1} \delta_m x_{im}^{\gamma_m} x_{ic}^{\gamma_j - \gamma_m} } , \label{g_equation1}
\end{equation}

\noindent for any observation $i \in \{ 1, \ldots, n \}$ and category  $j \in \{1 , \ldots c-1 \}$. The recalibrated probability prediction for the baseline category in observation $i$ is
\begin{equation}
    g_{ic} = 1 - \sum_{j=1}^{c-1} g_{ij}. \label{g_equation2}
\end{equation}

 Given any parameter values $\bm{\delta}, \bm{\gamma}$ and original probability predictions $\mathbf{X}$, the recalibrated predictions in $\mathbf{G}$ are directly calculated from Equations (\ref{g_equation1}) and (\ref{g_equation2}). A derivation is provided in the supplementary material.  Section \ref{sec:recalibration} describes a maximum likelihood-based approach to recalibrate a set of probability predictions $x_{ij}$ to $g_{ij}$.

\subsection{MCLLO Likelihood} 

Recall the random $n \times c$ matrix of observations   $\mathbf{Y}$ and  the random $n \times c$ matrix $\mathbf{X}$ of uncalibrated probability predictions on $c$ categories. The Multicategory Linear Log Odds Likelihood function is defined by

\begin{equation} \label{likelihood}
    \mathscr{L} ( \bm{\delta}, \bm{\gamma} ; \mathbf{X}, \mathbf{Y}) = \prod_{i=1}^n \prod_{j=1}^c  g_{ij}^{  Y_{ij}} ,
\end{equation}

\noindent where $g_{ij}$ is defined by Equations (\ref{g_equation1}) and (\ref{g_equation2}). This is a multinomial likelihood with recalibrated probabilities $g_{ij}$ as event probabilities.

\subsection{Recalibration of Probability Predictions} \label{sec:recalibration}

Equation (\ref{likelihood}) allows us to obtain maximum likelihood estimates (MLEs) $\Hat{\bm{\delta}}_{MLE}$ and $\Hat{\bm{\gamma}}_{MLE}$. We have found the Broyden\--Fletcher\--Goldfarb\--Shannon (BFGS) algorithm \citep{nocedal_bfgs} in \texttt{R} to be effective in calculating $\Hat{\bm{\delta}}_{MLE}$ and $\Hat{\bm{\gamma}}_{MLE}$. See the supplementary material for our reproducible code. Employing $\Hat{\bm{\delta}}_{MLE}$ and $\Hat{\bm{\gamma}}_{MLE}$ in Equations (\ref{g_equation1}) and (\ref{g_equation2}) allows a researcher to recalibrate probability predictions in the MCLLO framework.

\subsection{Likelihood Ratio Test} \label{section:LRT}

Our MCLLO approach includes a likelihood ratio test (LRT) to assess the plausibility that the probability predictions $x_{ij}$  are well calibrated. The hypotheses are
\begin{align*}
    H_0 &: \bm{\delta} = \bm{\gamma} = \mathbf{1}  \textnormal{ (probability predictions are well-calibrated)}, \text{versus} \\
    H_1 &: \delta_j \neq 1 \textnormal{ and/or } \gamma_j \neq 1 \textnormal{ for some } 1 \leq j \leq c-1.
\end{align*}

 The LRT statistic for $H_0$ is 

\begin{equation}
    \lambda_{LR} = - 2 \log \left( \frac{\mathscr{L} ( \bm{\delta} = \bm{\gamma} = \mathbf{1} | \mathbf{X}, \mathbf{Y} )  }{ \mathscr{L} (  \bm{\delta} = \Hat{\bm{\delta}}_{MLE} , \bm{\gamma} = \Hat{\bm{\gamma}}_{MLE} | \mathbf{X}, \mathbf{Y} ) } \right). \label{eqn:LRTteststat}
\end{equation}

\noindent The test statistic asymptotically follows a $\chi^2$ distribution with $2(c-1)$ degrees of freedom. 

\subsection{Theoretical Properties}

In this section we provide justification for the asymptotic consistency of the MLEs $\hat{\bm{\delta}}$ and $\hat{\bm{\gamma}}$. Lemma 1  asserts the convexity of the negative log of the MCLLO likelihood shown in Equation (\ref{likelihood}). This property of the likelihood equation is important as it guarantees that the Hessian matrix of the negative log likelihood equation with respect to the parameters is positive semi-definite, and therefore has non-negative eigenvalues. 

\begin{lemma}
     The negative log of the MCLLO likelihood in Equation (5) is convex in $\bm{\tau} = \log \bm{\delta}$ and $\bm{\gamma}$. \label{lemma}
\end{lemma}

\begin{proof}
See Appendix.
\end{proof}

Lemma 1 leads to the following theorem of the asymptotic consistency of the MLEs $\hat{\bm{\delta}}$ and $\hat{\bm{\gamma}}$. This characteristic of $\hat{\bm{\delta}}$ and $\hat{\bm{\gamma}}$ is especially important as the MLEs are estimates and thus it is necessary to show that $\hat{\bm{\delta}}$ and $\hat{\bm{\gamma}}$  converge to the true parameter values  ${\bm{\delta}}$ and ${\bm{\gamma}}$ as sample size $n$ becomes large.

\begin{theorem}
    The maximum likelihood estimators $\hat{\bm{\theta}}_n = ( \hat{\bm{\delta}_n}, \hat{\bm{\gamma}_n} )$  for the log of the MCLLO likelihood shown in Equation (5), $\ell ( \bm{\theta})$, are asymptotically consistent. In other words, 
    \begin{align*}
\hat{\bm{\theta}}_n &\xrightarrow{p} \bm{\theta},
\end{align*}
where $\hat{\bm{\theta}}_n = (\hat{\bm\delta}_n, \hat{\bm\gamma}_n)$ is the MLE based on a sample of size $n$, and $\bm{\theta}_0 = (\bm\delta_0,\bm \gamma_0)$ is the true parameter value. \label{theorem}
\end{theorem}

\begin{proof}
 See Appendix.
\end{proof}

\section{Simulation Study} \label{section:Simulation}

\subsection{Simulation Study Results: LRT} \label{section:simulationresults}

  We demonstrate how statistical power of our LRT is impacted by sample size $n$ and effect size (i.e. $\bm{\delta}$ and $\bm{\gamma}$'s distance from $\mathbf{1}$). This is implemented through a Monte Carlo simulation study with $c = 10$ categories. We work with thirty sets of simulated data, with sample sizes $n$ from the set $\{ 400, 450, 500, 600, 1000, 5000\}$, with shifts from perfectly well-calibrated data except for the shift parameter $\delta_1$ taking values from the set $\{ 1, 1.1, 1.2, 1.3, 1.4 \}$. More information about the simulation process can be found in the supplementary material.

 We perform our LRT described in Section \ref{section:LRT} and calculate a $p$-value for each repetition using $\alpha = 0.05$. The rejection rate for each simulation is plotted in Figure \ref{fig:simulationrejectionrate}. The results show that our LRT becomes more powerful as sample size and effect size grow.

\begin{figure}

\centering{

\includegraphics[width=3.5in,height=\textheight]{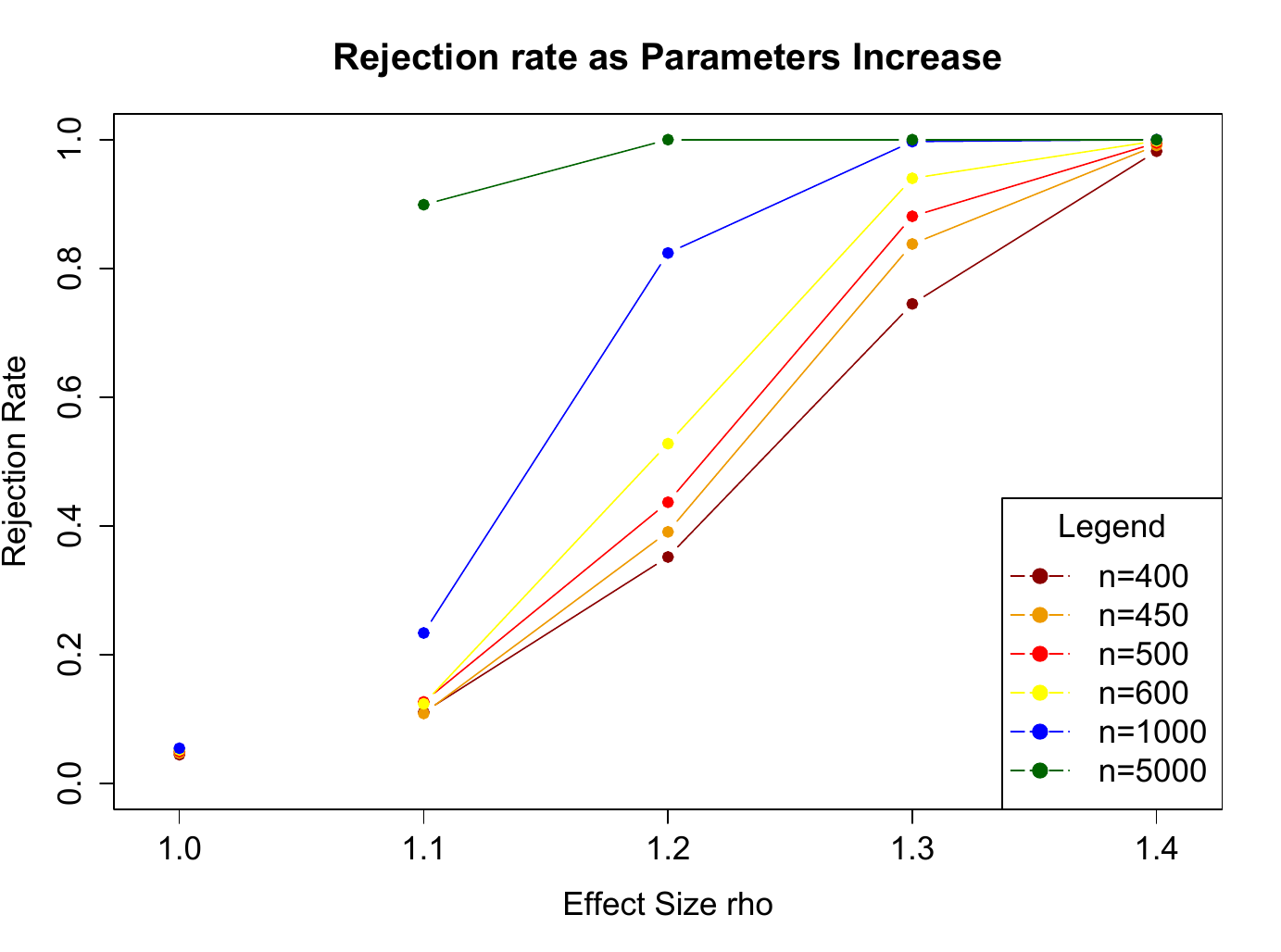}

}

\caption{\label{fig:simulationrejectionrate} Thirty simulations of 1000 Monte Carlo repetitions  were run with varying sample sizes and effect sizes. The rejection rates for each simulation are tabulated in the supplementary material and visualized in this figure, showing the effect of sample size and effect size on the power of our LRT. }

\end{figure}

\subsection{Simulation Study: ECE} \label{section:ECEsimulationresults}

To show how our recalibration strategy achieves calibration according to existing metrics (and not just our own LRT approach), we asses ECE \citep{guo2017calibration, Naeini_2015_bayesianbinning}, showing it improves as our MCLLO recalibration approach is applied.  We analyze the uncalibrated predictions $\mathbf{X}$ and MCLLO-recalibrated predictions $\mathbf{X}^*$. In each simulation, we calculate the ECE of $\mathbf{X}$ and $\mathbf{X}^*$, denoted $ECE(\mathbf{X})$ and $ECE(\mathbf{X}^*)$. Figure \ref{fig:ecegraph} display boxplots of $ECE(\mathbf{X})$ (in navy) and $ECE(\mathbf{X}^*)$ (in green) for the five simulations of sample size $n=5000$ with effect sizes $ \{ 1, 1.1, 1.2, 1.3, 1.4\}$. The boxplots show that though  the ECE of probability predictions $\mathbf{X}$ increases with greater effect size, the MCLLO recalibrated predictions $\mathbf{X}^*$ are more stable no matter how uncalibrated the original probability predictions $\mathbf{X}$ are. The supplementary material contains visualizations of the comparisons for the other simulations conducted with different sample sizes.

\begin{figure}
\centering{
\includegraphics[width=4.2in,height=\textheight]{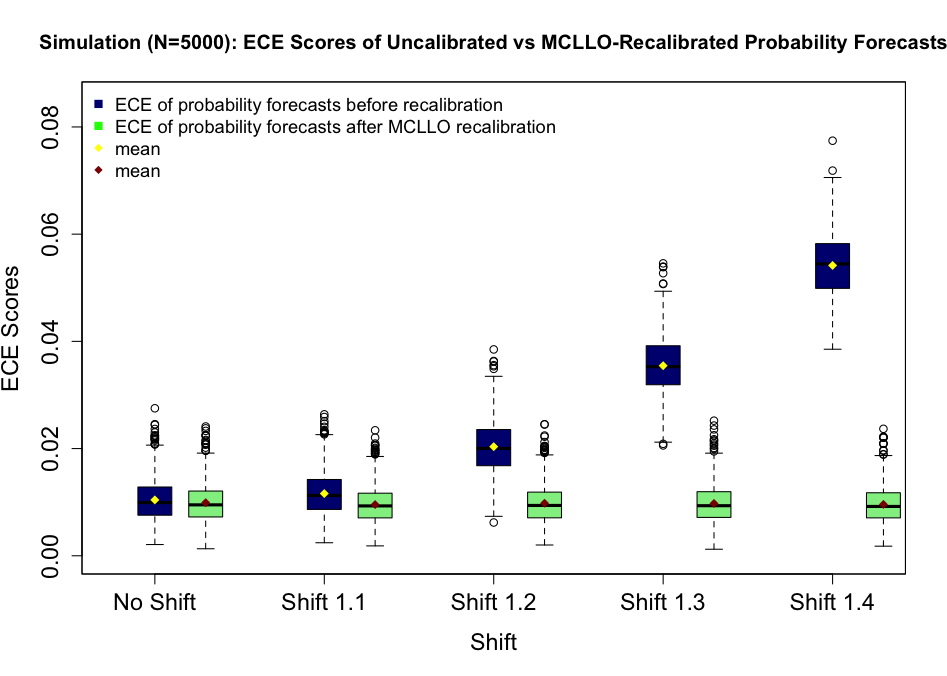}
}
\caption{\label{fig:eceboxplot} Five simulations of 1000 Monte Carlo Repetitions (with sample size $n=5000$) were run with varying effect sizes. The ECE of the original probability predictions $ECE(\mathbf{X})$ for each repetition are plotted in navy, and the ECE of the MCLLO recalibrated probability predictions $ECE(\mathbf{X}^*)$ for each repetition are plotted in green.}
\end{figure}

We also show ECE fluctuates as sample size and effect size vary, making articulation of meaningful ECE thresholds difficult. Figure \ref{fig:ecegraph} plots the mean difference in the ECE of the uncalibrated versus MCLLO-recalibrated probability predictions $ECE(\mathbf{X}) -ECE(\mathbf{X}^*)$ for the 1000 repetitions of twenty of the simulations. Figure \ref{fig:ecegraph} shows that MCLLO recalibration has little effect on the ECE of the uncalibrated probability predictions $\mathbf{X}$ when the data are not calibrated. As effect size grows, MCLLO recalibration increasingly lowers the ECE of the probability predictions. Unlike the LRT analysis in the previous subsection, the effectiveness of recalibration via ECE is not a function of sample size. 

\begin{figure}
\centering{
\includegraphics[width=4.5in,height=\textheight]{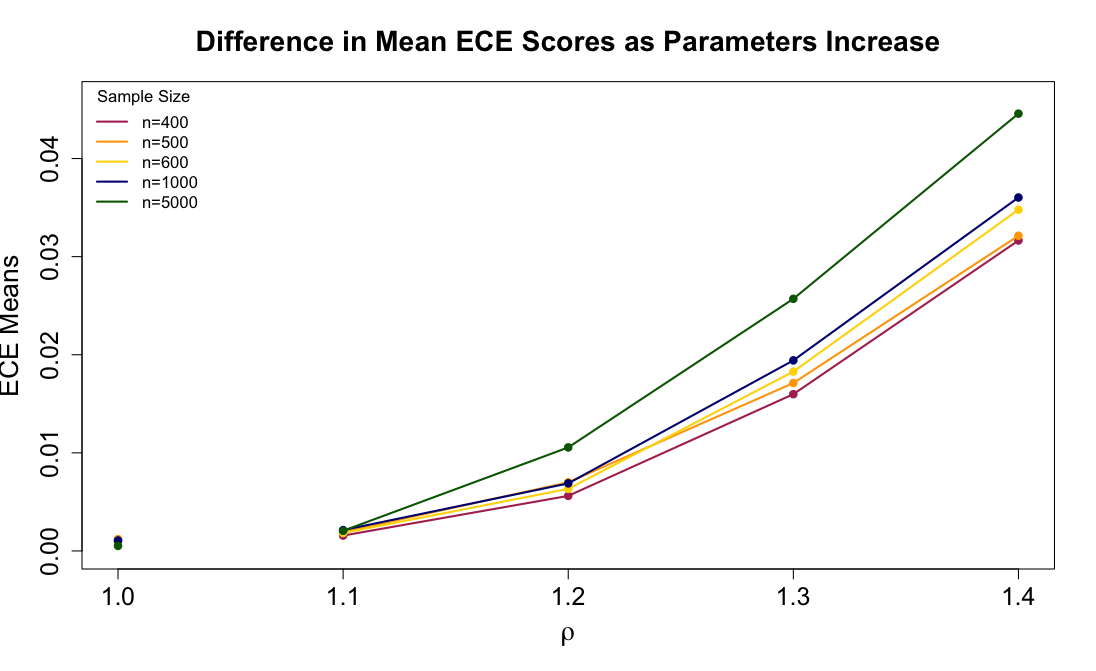}
}

\caption{\label{fig:ecegraph}Twenty simulations of 1000 Monte Carlo repetitions  were run with varying sample sizes and effect sizes. The mean difference in ECE scores $ECE(\mathbf{X}) - ECE(\mathbf{X}^*)$ for each simulation are visualized. }

\end{figure}

\section{Case Studies} \label{section:casestudies}

We demonstrate the capabilities of MCLLO recalibration through three case studies on image classification, obesity, plus an and ecology example in the supplementary material. In each example, we apply MCLLO recalibration to probability predictions and assess calibration using both ECE and the MCLLO LRT described in Section  \ref{section:LRT}. Note that in all three studies, we select the baseline category to be the last category alphabetically.  We compare MCLLO recalibration to four existing recalibration techniques in each setting.

The case studies showcase several properties of MCLLO recalibration that some recalibration techniques lack. First, they demonstrate the wide range of applications that MCLLO recalibration can be applied to, as it does not require under-the-hood access to the underlying model. Second, we show the interpretability of the MCLLO LRT relative to ECE. The LRT is designed specifically to assess whether a set of probability predictions is poorly calibrated, wheras a standalone ECE score is not. We also show that MCLLO recalibration does not depend on a user-specified bin size. Finally, we demonstrate that MCLLO recalibration produces probability predictions that are comparable to or more well-calibrated than those obtained from competing recalibration methods, as confirmed by both the MCLLO LRT and ECE.

\subsection{Comparator Methods}\label{ section:comparatormethods}

We compare MCLLO recalibration to four comparator methods. The first comparator method is a multiclass calibration technique developed by \cite{xudavoine} that utilizes a likelihood-based belief function to recalibrate a set of probability predictions. The second is a  binning procedure as described as a comparison for the methods proposed by \cite{guo2017calibration}. The final two comparator methods are temperature scaling \citep{hinton2015temperaturescaling} and vector scaling \citep{springerpaper}. Temperature scaling involves a single parameter for recalibration while vector scaling uses $2c$ parameters. Both scaling approaches require under-the-hood access to pre-softmax layer logits and thus cannot be used in the obesity case study which rely on random forests. More information on these methods can be found in the supplementary material.

\subsection{Case Study: Image Classification} \label{section:data:cifar}

We utilize the CIFAR-10 dataset \citep{Krizhevsky2009cifar}, which consists of $n=10,000$ images of objects that belong to one of of ten categories: \textit{airplane, automobile, bird, cat, deer, dog, frog, horse, ship,} or \textit{truck}. The observed categories for each observation  specified by CIFAR-10 are encoded in $\mathbf{Y}$. The set of probability predictions, $\mathbf{X}$, are outputted from VGG Net \citep{simonyan_2014_vggnet},a famous  image classification model. More details about our neural network implementation can be found in the supplementary material. The outputted probability predictions have $93.00 \%$ for the holdout set accuracy on basis of the most probable prediction. 

We randomly split the probability predictions $\mathbf{X}$ into an 80\% training set (with $n_t = 8,000$ observations) and a 20\% holdout set (with $n_h = 2,000$ observations). Denote the training set of probability predictions as $\mathbf{X}_t$ and holdout set of probability predictions as $\mathbf{X}_h$. The corresponding observed labels are denoted $\mathbf{Y}_t$ and $\mathbf{Y}_h$.

\subsubsection{Calibration Assessment for Image Classification} \label{section:cifarresults}

We compute MLEs and apply the MCLLO Likelihood Ratio Test described in Section \ref{section:LRT} to determine if the original, uncalibrated confidence scores $\mathbf{X}_{t} $ and $ \mathbf{X}_{h}$ and observed values encoded in $\mathbf{Y}_t$ and $\mathbf{Y}_h$ are well-calibrated.

These estimates are shown in Table \ref{table:cifar_parameters} for both the training [Top] and holdout [Bottom] sets of confidence scores. The standard errors for these estimates are also listed in parentheses. The standard errors were obtained by inverting the Hessian matrix outputted by the \texttt{optim()} function, applied to the negative of the log likelihood function  described by Equation (\ref{likelihood}). The square root of the diagonal entries of the inverse Hessian matrix are the standard errors of the $2(c-1)$ parameter estimates. 

To quantify the calibration of both sets, we conduct the MCLLO LRT to test the calibration status of $\mathbf{X}_t$, $\mathbf{Y}_t$ and  $\mathbf{X}_h$, $\mathbf{Y}_h$. The test statistic is calculated to be $\lambda_{LR,t} = 533$ on 18 degrees of freedom for the training set and $\lambda_{LR,h} = 123$ for the holdout set, yielding $p$-values of $<0.0001$ for both sets of uncalibrated confidence scores. We conclude that the confidence scores $\mathbf{X}_{t}, \mathbf{X}_{h}$ and observed values $\mathbf{Y}_t$ and $\mathbf{Y}_h$ are not well-calibrated (even within the training set) and would benefit from recalibration.

\begin{longtable}{@{}l*{9}{>{\centering\arraybackslash}p{0.085\textwidth}}@{}}
\caption{Maximum likelihood estimates $\hat{\bm{\delta}}$ and $\hat{\bm{\gamma}}$ for probability predictions and observed values for the CIFAR-10 case study training set $\mathbf{X}_t$ and holdout set $\mathbf{X}_h$ of confidence scores, with baseline category \textit{truck}. The standard errors for each of these parameter estimators are shown in parentheses. }
\label{table:cifar_parameters}\tabularnewline
\toprule
\endfirsthead

\toprule
\endhead

\bottomrule
\endlastfoot

\multicolumn{10}{@{}l@{}}{\textbf{Image Classification Parameter Estimates that Maximize Likelihood: $\mathbf{X}_t$}} \\
$\hat{\bm{\delta}}$ & 1.461 (0.315) & 0.710 (0.135) & 1.275 (0.346) & 1.966 (0.456) & 0.908 (0.306) & 3.095 (0.769) & 1.166 (0.370) & 1.785 (0.539) & 1.111 (0.240) \\
$\hat{\bm{\gamma}}$ & 0.705 (0.028) & 0.679 (0.043) & 0.692 (0.025) & 0.711 (0.022) & 0.732 (0.028) & 0.685 (0.021) & 0.659 (0.030) & 0.729 (0.031) & 0.621 (0.030) \\
\addlinespace

\multicolumn{10}{@{}l@{}}{\textbf{Image Classification Parameter Estimates that Maximize Likelihood: $\mathbf{X}_h$}} \\
$\hat{\bm{\delta}}$ & 0.693 (0.338) & 0.665 (0.211) & 0.912 (0.496) & 1.588 (0.708) & 0.312 (0.234) & 1.613 (0.824) & 1.037 (0.635) & 0.603 (0.439) & 0.711 (0.331) \\
$\hat{\bm{\gamma}}$ & 0.746 (0.061) & 0.539 (0.060) & 0.719 (0.052) & 0.740 (0.047) & 0.872 (0.073) & 0.727 (0.046) & 0.676 (0.060) & 0.768 (0.077) & 0.679 (0.075) \\
\end{longtable}

\subsubsection{Image Classification: MCLLO Recalibration} \label{sec:CIFAR.MCLLO.recalibration} 

We recalibrate the holdout set of probability predictions as follows. We apply Equations (\ref{g_equation1}) and (\ref{g_equation2}) to the to the holdout set $\mathbf{X}_h$ along with the estimated parameters learned from training set as shown in Table \ref{table:cifar_parameters}. The matrix $\mathbf{X}_{h,MCLLO}^*$ contains the recalibrated probabilities of the holdout set of confidence scores.

MCLLO recalibration of the confidence scores  $\mathbf{X}_{h}$ can be visualized with reliability diagrams \citep{springerpaper} created with 10 bins, in which well-calibrated confidence scores  are plotted closer to the $x=y$ line. Figure \ref{fig:CIFAR_reliability} displays reliability diagrams that were created using the uncalibrated confidence scores $\mathbf{X}_h$ [Top Left], where [Top Middle] was created using the MCLLO-recalibrated confidence scores $\mathbf{X}_{h,MCLLO}^*$. A quick comparison of the two plots shows that the reliability diagram for uncalibrated confidence scores in [Left] are slightly further away from the $x=y$ line than the reliability diagram for MCLLO recalibrated confidence scores.

 We calculate the ECE of $\mathbf{X}_{h, MCLLO}^*$ to be 0.021, improved from 0.035 for the uncalibrated confidence scores $\mathbf{X}_h$. Further, the MCLLO LRT for this set of MCLLO-recalibrated probability predictions produces a test statistic of $\lambda = 24.70$, on 18 degrees of freedom corresponding to a $p$-value of 0.133. We therefore conclude that applying MCLLO recalibration to the holdout set of confidence scores yielded well calibrated probability predictions.

\subsubsection{CIFAR: MCLLO versus Comparator Methods}

We compare the MCLLO recalibrated confidence scores $\mathbf{X}_{h, MCLLO}^*$ to the all the recalibration techniques discussed in Section \ref{ section:comparatormethods}, denoting $\mathbf{X}_{h, Xu}^*, \mathbf{X}_{h, EB}^*, \mathbf{X}_{h, TS}^*$ and $\mathbf{X}_{h, VS}^*$ for the technique by \cite{xudavoine}, the extension of binning as described by \cite{guo2017calibration}, temperature scaling \citep{hinton2015temperaturescaling}, and vector scaling \citep{springerpaper}, respectively. 

Table \ref{fig:comparison_charts_cifar} lists the ECE (calculated with $\sqrt{n_t}$ bins for the training set and $\sqrt{n_h}$ bins for the holdout set)  and LRT calibration metrics calculated from $\mathbf{X}_{h,MCLLO}^*$, $\mathbf{X^*}_{h,Xu}$, $\mathbf{X^*}_{h,EB}$, $ \mathbf{X^{*}}_{h,TS}$, and $ \mathbf{X^{*}}_{h,VS}$ for the image classification confidence scores. The recalibrated probabilities via all methods report slightly higher or equal accuracy compared to the original probability predictions $\mathbf{X}_{h}$. In terms of calibration, the ECE results show that all recalibration techniques reduce ECE when compared to the original, uncalibrated confidence scores $\mathbf{X}_h$, with the exception of the technique by \cite{xudavoine}. $\mathbf{X}_{h,MCLLO}^*$ reports an ECE of 0.021, comparable to that of $\mathbf{X}_{h,TS}^*$ and $\mathbf{X}_{h,VS}^*$ at 0.020. This result is somewhat surprising, as vector scaling and MCLLO recalibration utilize $2c$ and $2(c-1)$ parameters, respectively, where temperature scaling only uses a single parameter to recalibrate the confidence scores for all the categories. We would therefore not expect temperature scaling to calibrate as well as vector scaling or MCLLO recalibration. On the other hand, the LRT $p$-values report that $\mathbf{X}_{h,MCLLO}^*, \mathbf{X}_{h, VS}^*$ and $\mathbf{X}_{h,EB}^*$ are well-calibrated, while $\mathbf{X}_{h,TS}^*$ is not. This conclusion is more reasonable than the one based on ECE results, as MCLLO recalibration and vector scaling utilize more parameters than temperature scaling.

 Figure \ref{fig:CIFAR_reliability} also shows  reliability diagrams for the recalibrated confidence scores $\mathbf{X}_{h,TS}^*$ [top right], $\mathbf{X}_{h, VS}^*$ [bottom left], $\mathbf{X}_{h, EB}^*$ [bottom middle], and $\mathbf{X}_{h, Xu}^*$ [bottom right], along with their corresponding ECE and LRT $p$-value from Table \ref{fig:comparison_charts_cifar}. It is not surprising that MCLLO and vector scaling produce confidence scores that are similarly calibrated as the two methods maximize a likelihood with respect to a similar number of parameters. From the reliability diagrams, it is clear that temperature scaling achieved remarkable calibration at a low cost of complexity: the reliability diagram representing $\mathbf{X}_{h,TS}^*$ follows the $x=y$ line and its ECE is comparable to more complex methods like MCLLO and vector scaling. However, $p$-values report that temperature scaling does not provide well-calibrated predictions. An examination of which parameter estimates used to calculate these $p$-values differ ($\bm{\delta}_4$ and $\bm{\delta}_6$) indicate that those respective categories are poorly handled by temperature scaling. Thus, the MCLLO approach shows that temperature scaling handles the \textit{dog} and \textit{cat} category poorly, with those specific categories driving the termperature scaling results towards a lack of calibraiton. See the supplementary material for more details.

\begin{figure}

\centering{

\includegraphics[height=\textheight]{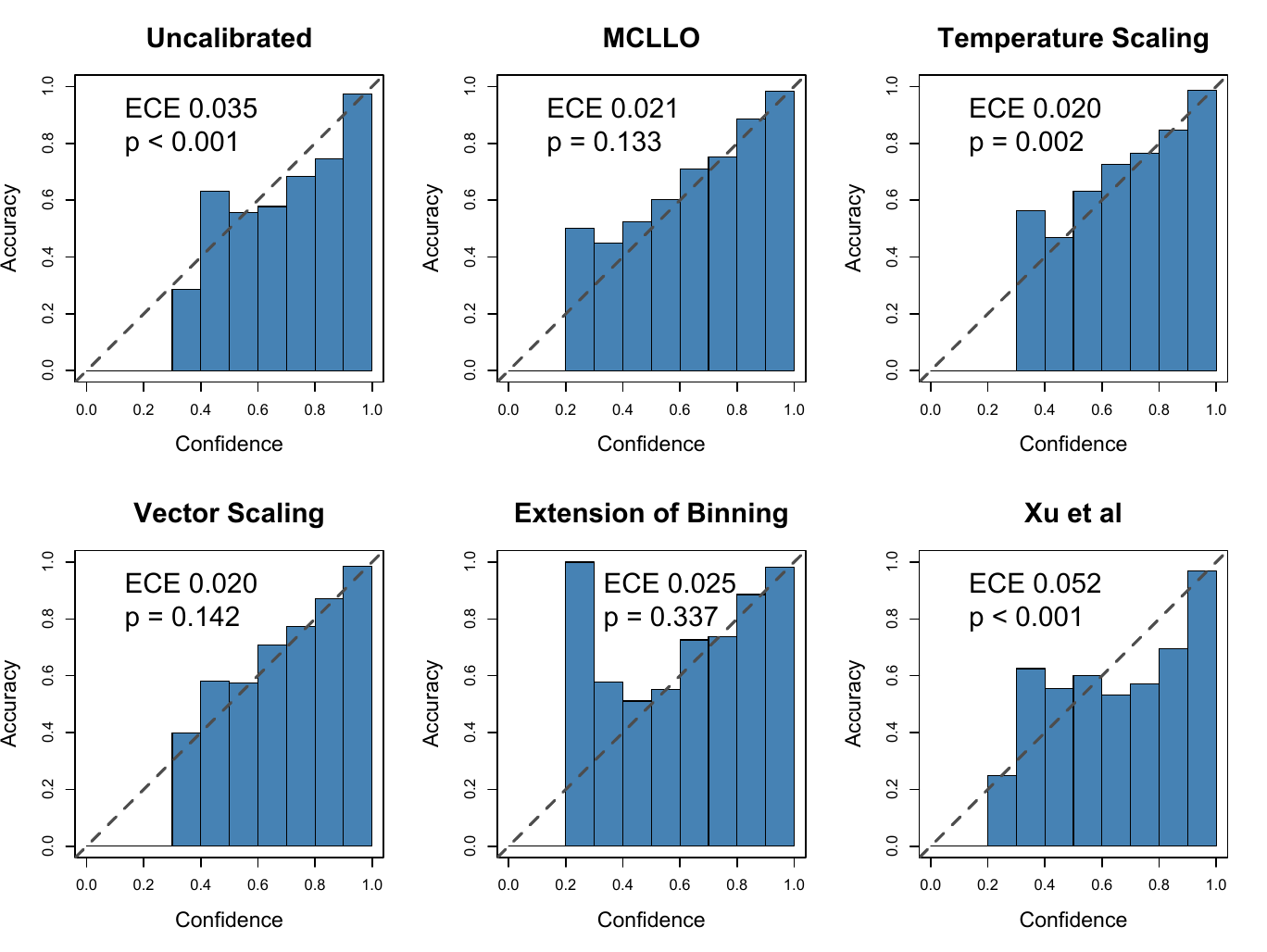}

}

\caption{\label{fig:CIFAR_reliability} Reliability diagrams for the CIFAR holdout confidence scores before [top left] and after different methods of recalibration, where well-calibrated predictions lie along the $x=y$ line. The top row shows (left to right) the original uncalibrated holdout confidence scores $\mathbf{X}_h$, MCLLO-recalibrated confidence scores $\mathbf{X}_{h,\text{MCLLO}}^*$, and temperature scaled confidence scores $\mathbf{X}_{h, TS}^*$.	The bottom row shows (left to right) vector scaled confidence scores $\mathbf{X}_{h,VS}^*$, confidence scores obtained through the extension of binning as described by \cite{guo2017calibration}, denoted $\mathbf{X}_{h, EB}^*$, and those obtained via the recalibration method by \cite{xudavoine}, denoted $\mathbf{X}_{h,Xu}^*$. }

\end{figure}

\begin{longtable}{@{}lccccc@{}}
\caption{Comparison of calibration metrics from applying recalibration techniques on CIFAR-10 confidence scores. Reported are the $p$-value from MCLLO’s likelihood ratio test, the ECE score with $\sqrt{n}$ bins, accuracy, and the percentage of first-place labels that changed after recalibration.}
\label{fig:comparison_charts_cifar}\tabularnewline
\toprule
Calibration set & $n$ & MCLLO $p$-value & ECE ($\sqrt{n}$ bins) & Accuracy & \% label change \\
\midrule
\endfirsthead

\toprule
Calibration set & $n$ & MCLLO $p$-value & ECE ($\sqrt{n}$ bins) & Accuracy & \% label change \\
\midrule
\endhead

\bottomrule
\endlastfoot

$\mathbf{X}_{t}$ & 8000 & $<0.001$ & 0.033 & 92.7\% & -- \\
$\mathbf{X}_{h}$ & 2000 & $<0.001$ & 0.035 & 93.0\% & -- \\
$\mathbf{X}_{h,\text{MCLLO}}^{*}$ & 2000 & 0.133 & 0.021 & 93.1\% & 1.95\% \\
$\mathbf{X}_{h,\text{Xu}}^{*}$ & 2000 & $<0.001$ & 0.052 & 93.0\% & 0\% \\
$\mathbf{X}_{h,\text{EB}}^{*}$ & 2000 & 0.337 & 0.025 & 93.2\% & 0.8\% \\
$\mathbf{X}_{h,\text{TS}}^{*}$ & 2000 & 0.002 & 0.020 & 93.0\% & 0\% \\
$\mathbf{X}_{h,\text{VS}}^{*}$ & 2000 & 0.142 & 0.020 & 93.1\% & 1.8\% \\
\end{longtable}

\subsection{Case Study: Obesity} \label{section:data:obesity}

\subsubsection{Random forest as predictive model}

We obtain probability predictions by fitting a random forest to a publicly available data set of 2,111 participants that examined 7 obesity classifications as a function of participant characteristics. We used a 75\% training 25\% holdout split for this analysis, resulting in probability predictions $\mathbf{X}$ for $n=528$ individuals. More detail on the implementation of this random forest can be found in the supplementary material.

\subsubsection{Recalibration of Obesity Case Study Probability Predictions}

 We recalibrate the probability predictions $\mathbf{X}$ as follows. We split $\mathbf{X}$ into an additional 75\% training and 25\% holdout split, resulting in $\mathbf{X}_{t}, \mathbf{X}_{h},$ (each with $n_{t} = 396,$ and $n_{h} = 132$ observations). We then obtain MCLLO MLEs, $\hat{\bm{\delta}}$ and $\hat{\bm{\gamma}}$, learned from the training set $\mathbf{X}_{t}$, which are displayed in Table (\ref{table:obesity_parameters}) with baseline category ``Normal Weight."  The standard errors for for these estimates, calculated similarly to those calculated in Section \ref{sec:CIFAR.MCLLO.recalibration}, are listed in parentheses. Finally, we apply MCLLO recalibration with baseline category ``Normal Weight" with the learned parameters to the holdout set $\mathbf{X}_{h}$ : in other words, we apply Equations (\ref{g_equation1}) and (\ref{g_equation2}) to the parameters in Table (\ref{table:obesity_parameters}) and the holdout set, to calculate recalibrated probabilities that are stored in the matrix $\mathbf{X}_{h, MCLLO}^*$.

\begin{longtable}{@{}l*{6}{>{\centering\arraybackslash}p{0.11\textwidth}}@{}}
\caption{ Maximum likelihood estimates $\hat{\bm{\delta}}$ and $\hat{\bm{\gamma}}$ for probability predictions and observed values for the obesity case study, with baseline category \textit{Normal Weight}. The standard errors for these estimates are also provided in parentheses. }
\label{table:obesity_parameters}\tabularnewline
\toprule
\endfirsthead

\toprule
\endhead

\bottomrule
\endlastfoot

\multicolumn{7}{@{}l@{}}{\textbf{Obesity Parameter Estimates that Maximize Likelihood: $\mathbf{X}_{t}$}} \\
$\hat{\bm{\delta}}$ & 0.527 (0.241) & 1.002 (0.406) & 2.714 (1.404) & 0.002 (0.001) & 1.167 (0.391) & 1.529 (0.524) \\
$\hat{\bm{\gamma}}$ & 2.785 (0.399) & 2.211 (0.214) & 1.719 (0.175) & 4.698 (0.960) & 2.387 (0.247) & 2.020 (0.188) \\
\end{longtable}

We similarly recalibrate the holdout probability predictions with the  technique by \cite{xudavoine} and the extension of binning as described by \cite{guo2017calibration}, learning parameter estimates on the training set and applying those estimates to the holdout set. This results in the recalibrated probability predictions $\mathbf{X}_{h, Xu}^* $ and  and $\mathbf{X}_{h, EB}^*$. Since the uncalibrated probability predictions are outputs from a random forest and pre-softmax logits are not attainable, temperature scaling and vector scaling do not apply and we can therefore not include recalibrated probability predictions via these methods in our analysis. 

\subsubsection{Results and Calibration Metrics of Obesity Probability Predictions}

Figure \ref{fig:obesity_reliability} displays the reliability diagrams \citep{springerpaper}, created with 10 bins, for $\mathbf{X}_{h}, \mathbf{X}_{h,MCLLO}^*, \mathbf{X}_{h, EB}^*,$ and $\mathbf{X}_{h,Xu}^*$, along with their corresponding ECE and $p$-values from the MCLLO LRT. Visually, the recalibrated probability predictions follow the $x=y$ line more closely than $\mathbf{X}_h$, which agrees with the displayed calibration metrics.

\begin{figure}
\centering{
\includegraphics[height=\textheight]{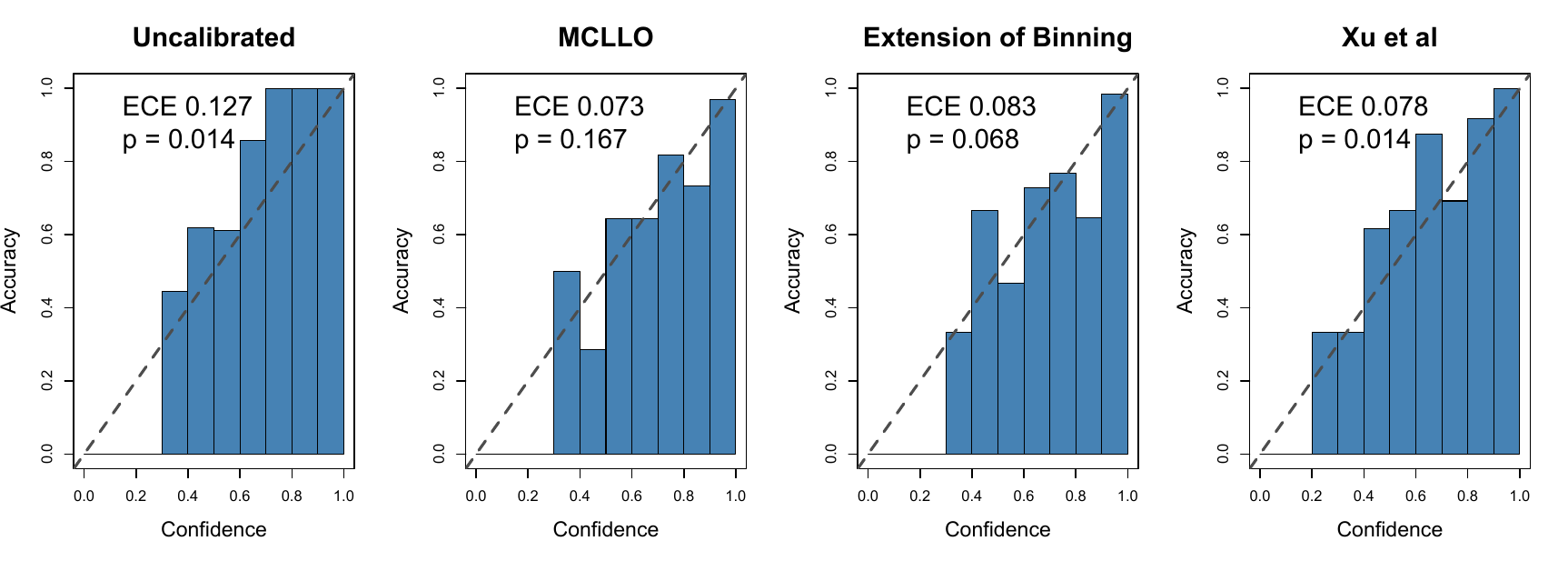}
}
\caption{\label{fig:obesity_reliability} Reliability diagrams for the obesity holdout confidence scores before [top left] and after different methods of recalibration, where well-calibrated predictions lie along the $x=y$ line. [Left] to [Right] displays diagrams for the uncalibrated holdout confidence scores $\mathbf{X}_h$, and recalibrated confidence scores $\mathbf{X}_{h,\text{MCLLO}}^*$, $\mathbf{X}_{h, EB}^*$, and $\mathbf{X}_{h,Xu}^*$. }
\end{figure}

Table (\ref{tab:obesity:calibration.metrics}) lists the calibration metrics computed for the uncalibrated and recalibrated (via MCLLO and comparator methods) probability predictions. For all sets of probability predictions, we report ECE with 20 or 12 bins (for the training and holdout sets, respectively) and MCLLO $p$-value. The MCLLO LRT concludes that the uncalibrated holdout probability predictions $\mathbf{X}_{h}$ are not well-calibrated for a traditional $\alpha = 0.05$ decision rule. The extension of binning as described by \cite{guo2017calibration} and, as expected, MCLLO recalibration, produce well-calibrated probability predictions via the MCLLO LRT. MCLLO recalibration also results in the lowest ECE than the competitor recalibration techniques. It seems noteworthy that the availability of MCLLO LRT can identify successful recalibration based on other methods. Our LRT approach can help researchers learn which recalibration strategy may be best in a given situation. Also, see the supplementary material for a discussion of the tradeoff between accuracy and calibration.

\begin{longtable}{@{}lccccc@{}}
\caption{Calibration metrics for the obesity case study. Probability predictions were obtained by applying a random forest to the obesity data set and recalibrated using MCLLO, the method of Xu et al., and an extension of binning described by Guo et al. Reported are the $p$-values from MCLLO’s likelihood ratio test, the ECE score with $\sqrt{n}$ bins, accuracy, and the percentage of first-place predicted labels that changed after recalibration.}
\label{tab:obesity:calibration.metrics}\tabularnewline
\toprule
Calibration set & $n$ & MCLLO $p$-value & ECE ($\sqrt{n}$ bins) & Accuracy & \% label change \\
\midrule
\endfirsthead

\toprule
Calibration set & $n$ & MCLLO $p$-value & ECE ($\sqrt{n}$ bins) & Accuracy & \% label change \\
\midrule
\endhead

\bottomrule
\endlastfoot

$\mathbf{X}_{t}$ & 396 & $<0.001$ & 0.160 & 85.1\% & -- \\
$\mathbf{X}_{h}$ & 132 & 0.014 & 0.127 & 78.8\% & -- \\
$\mathbf{X}_{h, \text{MCLLO}}^{*}$ & 132 & 0.167 & 0.073 & 81.1\% & 2.27\% \\
$\mathbf{X}_{h, \text{EB}}^{*}$ & 132 & 0.068 & 0.083 & 81.0\% & 3.03\% \\
$\mathbf{X}_{h,\text{Xu}}^{*}$ & 132 & 0.014 & 0.078 & 81.1\% & 3.0\% \\
\end{longtable}

\section{Discussion and Conclusion} \label{sec:discussion}

We have shown that the proposed MCLLO approach is a valuable contribution to the literature because it (i) directly tests the calibration of a single model and furnishes a recalibration strategy for faulty model predictions, (ii) does not require under-the-hood model access, e.g., accessing logit-scale predictions within the layers of a neural network, and (iii) provides output which is easy for human analysts to understand. This is important because reliable probability predictions in the multicategory realm are crucial for making effective decisions. For example, classifying an image as a ``plane" with a 40\% probability prediction versus a 60\% probability prediction may cause humans to make different decisions in this case. 

Though the general idea of hypothesis testing to identify well-calibrated probability predictions has been mentioned in the literature by \cite{Vaicenavicius2019multiclasshypothesistest}, previous literature does not fully describe the methodology of conducting a statistical hypothesis test. In particular, previous work does not explicitly define a likelihood function, a null hypothesis equivalent to well-calibrated probability predictions, or methodology to obtain a specific test statistic or distribution needed to calculate a $p$-value. Additionally, \cite{Vaicenavicius2019multiclasshypothesistest} describes assessing calibration but does not offer a solution to recalibrate a set of probability predictions. In contrast, the proposed MCLLO approach is, to our knowledge, the first fully described and demonstrated method to test whether or not a set of probability predictions are plausibly well calibrated. Further, MCLLO MLEs can be used to recalibrate faulty predictions in the event that the corresponding test rejects calibration.

The MCLLO LRT has some specific advantages over other popular metrics of calibration like ECE. First, the  LRT test statistic in Equation (\ref{eqn:LRTteststat}) controls the type I error rate and gains power as sample size or departure from calibration grows, as shown in the simulation study in the supplementary material. Second, the interpretation of ECE is limited: though lower ECE scores signify better calibration, ECE alone cannot determine with any level of significance if a set of probability predictions is well-calibrated or not. The MCLLO LRT can direct test the calibration hypothesis. Additionally, ECE does not take into account every probability prediction: instead, it only considers the maximum probability prediction of each observation, while the MCLLO LRT utilizes the probability predictions for all the categories for each observation. Finally, ECE does not prescribe a strategy to recalibrate predictions, while MCLLO does. ECE depends on a user-specified number of bins, and different choices of bin number can lead to different ECE scores. MCLLO recalibration does not require binning, thus removing one more subjective specification that may influence results. Despite these drawbacks, we note ECE is a useful and popular metric, and we have shown (see supplementary material) that our MCLLO recalibration approach also improves ECE. The metrics can be used in concert to gain the advantages of each.

The MCLLO approach includes identity map. If $\bm{\delta}=\bm{\gamma}=\bm{1}$, then $\bm{X}=\bm{G}$ and the original predictions are unchanged. I.e., an MCLLO recalibration can conceptually leave well-calibrated predictions alone. This feature is not present in some existing recalibration techniques. For example, the recalibration method by \cite{xudavoine} described in Section \ref{section:casestudies} does not include the identity map. Thus, even well calibrated predictions will be adjusted at least slightly by the \cite{xudavoine} method.

In addition to performing comparatively, and sometimes better, than temperature and vector scaling, we have demonstrated in Section \ref{section:data:obesity} and the supplementary material that MCLLO recalibration applies to a broader set of applications than temperature and vector scaling, as it does not require under-the-hood access to the logits of a model that generates probability predictions. For example, MCLLO recalibration can be used for random forests where temperature and vector scaling cannot. This is important, as our case study shows random forests are not guaranteed to produce well calibrated probability predictions (Table \ref{tab:obesity:calibration.metrics}). Thus the case studies exhibit that if a predictive model does not produce well calibrated probability predictions, MCLLO recalibration can map the outputted predictions to well-calibrated probability predictions.

We close by addressing a limitation of MCLLO recalibration: in other multicategory models, like multinomial logistic regression \citep{agresti2006multinomiallogisticregression}, baseline category choice is completely arbitrary, meaning that the likelihood, MLEs, recalibrated probabilities, and testing results are invariant to the choice of baseline category. Perhaps surprisingly, this property does not extend to MCLLO recalibration when the alternative hypothesis $H_1$ is true. The shift parameter $\bm{\delta}$  of the MCLLO formula (Equation (\ref{logodds_linear_model})) under one baseline category choice cannot be directly represented as a function of the parameters of another baseline category choice. Thus, it is possible that specification of the baseline category could affect results. Empirically, we have found these effects to be very small (see supplementary material).

\section{Disclosure statement}\label{disclosure-statement}

The authors declare no conflict of interest exists.

\section{Data Availability Statement}\label{data-availability-statement}

Data have been made available at the following URL: \textcolor{red}{This will be established upon article acceptance.}

\phantomsection\label{supplementary-material}
\bigskip

\bibliography{bibliography.bib}

\section{Appendix}

\subsection{Proofs of Lemmas and Theorems}

\begin{lemma}
     The negative log of the MCLLO likelihood in Equation (5) is convex in $\bm{\tau} = \log \bm{\delta}$ and $\bm{\gamma}$. \label{lemma}
\end{lemma}

\begin{proof}
Choose an arbitrary $i \in \{ 1, \ldots , n\}$. Define $ W = \left[
\begin{array}{c|c}  
\bm{\tau} & \text{Diag } \bm{\gamma}
\end{array}
\right]$
and $\mathbf{x} = \begin{bmatrix}
    1 &  \log \left( \frac{x_{i1}}{x_{ic}} \right) & \cdots &  \log \left( \frac{x_{i c-1}}{x_{ic}} \right)
\end{bmatrix}^\top$, resulting in the activation vector $\mathbf{a} = W \mathbf{x} = \begin{bmatrix}
    \log \left( \frac{g_{i1}}{g_{ic}} \right) & \cdots & \log \left( \frac{g_{i c-1}}{g_{ic}} \right) 
\end{bmatrix}^\top $. Then $\mathbf{y} = \text{softmax} ( \mathbf{a} ) = \begin{bmatrix}
    \frac{g_{i1}}{1- g_{ic}} & \cdots & \frac{g_{i c-1}}{1 - g_{ic}}
\end{bmatrix}^\top $, is a parallel from \cite{rychlik_2021_psdHessian} to MCLLO recalibration. By Corollary 3 in \cite{rychlik_2021_psdHessian}, the negative log of the MCLLO likelihood is convex in $\bm{\tau}$ and $\bm{\gamma}$.
\end{proof}

\begin{theorem}
    The maximum likelihood estimators $\hat{\bm{\theta}}_n = ( \hat{\bm{\delta}_n}, \hat{\bm{\gamma}_n} )$  for the log of the MCLLO likelihood shown in Equation (5), $\ell ( \bm{\theta})$, are asymptotically consistent. In other words, 
    \begin{align*}
\hat{\bm{\theta}}_n &\xrightarrow{p} \bm{\theta},
\end{align*}
where $\hat{\bm{\theta}}_n = (\hat{\bm\delta}_n, \hat{\bm\gamma}_n)$ is the MLE based on a sample of size $n$, and $\bm{\theta}_0 = (\bm\delta_0,\bm \gamma_0)$ is the true parameter value. \label{theorem}
\end{theorem}

\begin{proof}
We verify two assumptions required for this proof. First, the negative log of the MCLLO likelihood function is
\begin{align}
   - \ell ( \bm{\theta}) &= - \sum_{j=1}^c \sum_{i=1}^n y_{ij} \log \delta_j - \sum_{j=1}^c \sum_{i=1}^n y_{ij} \gamma_j \log x_{ij} - \sum_{j=1}^c \sum_{i=1}^n y_{ij} \left( \sum_{z=1}^c \gamma_z - \gamma_j \right) \log x_{ic} \\
    & + \sum_{i=1}^n \left[ \left( \sum_{j=1}^n y_{ij} \right) \log \left( \sum_{m=1}^c \delta_m x_{im}^{\gamma_m} x_{ic}^{\sum_{z=1}^c \gamma_z - \gamma_m} \right) \right] ,
\end{align}

\noindent and therefore is twice continuously differentiable. Second, the expected value of the score function (i.e., the gradient of the negative log-likelihood function) at the true parameter value is zero: $\mathbb{E}\left[ - \frac{\partial \ell(\theta_0)}{\partial \theta}\right] = 0$. This is a property of the likelihood function as an objective function that has a minimum value.

By definition of a minimum, which is numerically approximated by the Broyden–Fletcher–Goldfarb–Shannon (BFGS) algorithm \cite{nocedal_bfgs} , the MLE $\hat{\bm \theta}_n$ satisfies the first-order condition:

\begin{align*}
- \frac{\partial \ell(\hat{\bm \theta}_n)}{\partial \bm \theta} &= 0.
\end{align*}

\noindent Using a Taylor series expansion of the score function around the true parameter value $\bm \theta_0$ and evaluating it at the MLE $\hat{\bm \theta}_n$, 
\begin{align*}
0 &= - \frac{\partial \ell(\hat{\bm \theta}_n)}{\partial \bm \theta} \\
&= - \frac{\partial \ell(\bm \theta_0)}{\partial \bm \theta} - \frac{\partial^2 \ell(\bm \theta^\ast)}{\partial  \bm\theta \partial \bm \theta^T}(\hat{\bm \theta}_n -\bm \theta_0)
\end{align*}
where $\bm \theta^*$ is a point between $\hat{\bm \theta}_n$ and $\bm \theta_0$, $\frac{\partial^2 \ell( \bm \theta^\ast)}{\partial \bm \theta \partial \bm \theta^T}$ is a matrix with $ij$th entry $\frac{\partial^2 \ell(\bm \theta^\ast)}{\partial \bm \theta_i \partial \bm \theta_j}$.

Rearrange the terms and multiply both sides by $\sqrt{n}$. If the Hessian matrix $- \frac{\partial^2 \ell(\bm\theta)}{\partial \theta \partial \theta^T}$ of $- \ell (\bm{\theta})$ is positive definite, use the inverse. Otherwise, by Lemma \ref{lemma}  the Hessian matrix is positive semi-definite and we use the generalized inverse to obtain
\begin{align*}
\sqrt{n}(\hat{\bm\theta}_n - \bm\theta_0) &= -\left(\frac{1}{n}\frac{\partial^2 \ell(\bm\theta)}{\partial \bm\theta \partial \bm\theta^T}\right)^{-}\left(\frac{1}{\sqrt{n}}\frac{\partial \ell(\bm\theta_0)}{\partial \bm\theta}\right)
\end{align*}

By Lemma \ref{lemma}, the Hessian matrix $- \frac{\partial^2 \ell(\bm\theta)}{\partial \bm\theta \partial \bm\theta^T}$ is positive semi-definite, which implies that all of its eigenvalues are non-negative. This justifies that $- \frac{1}{n}\frac{\partial^2 \ell(\bm\theta)}{\partial \bm\theta \partial \bm\theta^T}$ converges to the  Fisher information matrix $-\mathcal{I}(\bm\theta_0)$ as the non-negative values along the diagonal of this matrix are valid estimates for the variances of the parameters. Thus by the Law of Large Numbers (LLN),
\begin{align*}
\frac{1}{n}\frac{\partial^2 \ell(\bm\theta)}{\partial \bm\theta \partial \bm\theta^T} &\xrightarrow{p} -\mathcal{I}(\bm\theta_0)
\end{align*}

and the Central Limit Theorem (CLT),
\begin{align*}
   - \frac{1}{\sqrt{n}}\frac{\partial \ell(\bm\theta_0)}{\partial \bm\theta} &\xrightarrow{d} \mathcal{N}(0, \mathcal{I}(\bm\theta_0)).
\end{align*}

Applying Slutsky's theorem,
\begin{align*}
- \sqrt{n}(\hat{\bm\theta}_n - \bm\theta_0) &\xrightarrow{d} \mathcal{N}(0, \mathcal{I}^{-1}(\bm\theta_0)),
\end{align*}
which implies that $\hat{\bm\theta}_n \xrightarrow{p} \bm\theta_0$.
\end{proof}

\end{document}